\documentclass[draft,12pt,cleveref]{colt2025} 

\title[Confidence Estimation via Sequential Likelihood Mixing]{Confidence Estimation via Sequential Likelihood Mixing}
\usepackage{times}

\usepackage{mystyle}
\usepackage{algorithm2e}
\usepackage{subcaption}
\usepackage{caption}

\coltauthor{%
 \Name{Johannes Kirschner} \Email{johannes.kirschner@sdsc.ethz.ch}\\
 \addr Swiss Data Science Center
 \AND
 \Name{Andreas Krause} \Email{krausea@ethz.ch}\\
 \addr ETH Zurich, Department of Computer Science%
 \AND
 \Name{Michele Meziu} \Email{mezium@student.ethz.ch}\\
 \addr ETH Zurich, Department of Computer Science%
 \AND
 \Name{Mojmir Mutný} \Email{mmutny@inf.ethz.ch}\\
 \addr ETH Zurich, Department of Computer Science%
}

\usepackage[textsize=tiny]{todonotes}
\usepackage{layouts}
\newtheorem{assumption}{Assumption}

\begin{document}

\maketitle

\begin{abstract}%
We present a universal framework for constructing confidence sets based on sequential likelihood mixing. Building upon classical results from sequential analysis, we provide a unifying perspective on several recent lines of work, and establish fundamental connections between sequential mixing, Bayesian inference and regret inequalities from online estimation. The framework applies to any realizable family of likelihood functions and allows for non-i.i.d.~data and anytime validity. Moreover, the framework seamlessly integrates standard approximate inference techniques, such as variational inference and sampling-based methods, and extends to misspecified model classes, while preserving provable coverage guarantees. We illustrate the power of the framework by deriving tighter confidence sequences for classical settings, including sequential linear regression and sparse estimation, with simplified proofs.
\end{abstract}


\section{Introduction}
Estimating model uncertainty is a fundamental challenge in machine learning and data science. Uncertainty quantification is essential for assessing model accuracy and robustness, particularly in safety-critical domains such as medical diagnosis, drug discovery, and autonomous driving. Uncertainty quantification also plays a key role in guiding exploration in sequential decision-making algorithms, including active learning, adaptive experimental design and reinforcement learning. 

Significant efforts have been devoted to developing scalable uncertainty estimation techniques for modern machine learning applications. However, most existing work on uncertainty quantification focuses on narrow function classes and specific noise distributions, or relies on asymptotic expansions or adhoc approximations. Moreover, AI systems are often deployed in interactive settings, where data is generated sequentially, introducing complex data dependencies that render any i.i.d.~assumption invalid. Constructing practical, non-asymptotic, any-time valid confidence sets under universal assumptions on the stochastic model class remains a major challenge.

Classical uncertainty quantification methods originate in the field of statistics and are typically classified into Bayesian and frequentist approaches. In frequentist uncertainty estimation, only the data is random, imposing a distribution on the estimator constructed from the data. Confidence sets are derived by modeling the randomness of the data generating process, most commonly by specifying tail bounds on the error distribution. This approach relies on propagating the noise distribution through the estimator, often (though not always) leveraging central limit arguments, i.i.d. assumptions, and strong regularity conditions on the model class to ensure tractability. Moreover, this approach is invariably dependent on the estimator that one chooses. 

Bayesian statistics, on the other hand, introduces a prior distribution over the unknown parameter and updates the posterior using Bayes' rule and the likelihood. The posterior distribution over the model parameters can be interpreted as a measure of uncertainty. Notably, Bayesian inference is a `universal recipe', in the sense that it does not require i.i.d.~data or specific assumptions on the likelihood. Obtaining the Bayesian posterior is therefore a purely computational problem, and much work has been done to develop computationally efficient methods and approximations. 

Despite its popularity, the Bayesian approach is often referred to as \emph{subjective}, as it depends on the choice of the prior. In fact, it is well understood that Bayesian credible sets do not attain frequentist coverage in general, and arguably, the Bayesian view-point is only viable assuming that the true parameter is randomly sampled from the prior distribution. It therefore seems to be a common misconception that Bayesian inference cannot easily lead to confidence sets when frequentist coverage is required. As we will see soon, the opposite is true: 
\begin{quotation}
    \emph{Any Bayesian inference model can be turned into a frequentist confidence set, and a frequentist confidence set can be constructed using Bayesian inference, preserving provable coverage even if approximate inference is used or the model is misspecified.} 
\end{quotation}
In this work, we revisit classical ideas from sequential analysis \citep{wald1947sequential2} for constructing confidence sets using likelihood ratio martingales and Ville's inequality \citep{ville1939etude}, tracing back to early works by \citet{darling1968some,robbins1970statistical}. In its simplest form, the confidence sets are level sets of the negative log-likelihood function $L_t(\theta) = L_t(\theta|x_1,y_1,\dots, x_t, y_t)$,
\begin{align}
    C_t = \big\{\theta \in \Theta : L_t(\theta) \leq \beta_t(\delta)\big\} \,, \label{eq:level-set}
\end{align}
with a confidence coefficient $\beta_t(\delta)$ that scales with the failure probability $\delta \in (0,1)$, and the complexity of the model class. To build intuition, a typical confidence coefficient will take the following approximate form, 
$$\beta_t(\delta) \approx \log \frac{1}{\delta} + L_t(\hat \theta_t^{\MLE}) + B_t\,,$$ where $\hat \theta_t^{\MLE}$ is the maximum likelihood estimator, and $B_t \geq 0$ is a complexity term that depends on the model class. Written like this, the confidence coefficient defines a \emph{relative likelihood} that allows us to interpret the confidence level of a candidate model $\theta \in \Theta$ by comparing $L_t(\theta)$ to $L_t(\hat \theta_t^\MLE)$. In practice, however, a worst-case bound on the model complexity might lead to overly conservative confidence sets, or tight bounds may not be easily available for complex function classes. Therefore, it is beneficial to define the confidence coefficient $\beta_t(\delta)$ in a data-dependent way. Importantly, a smaller $\beta_t(\delta)$ directly translates into a tighter confidence set. 

A key insight is the use of sequential mixing distributions over the likelihood ratio martingale (formally introduced in \cref{sec:mixing}), which allows to define a data-dependent confidence coefficient applicable to general model classes. Moreover, sequential mixing establishes a fundamental connection between Bayesian inference and frequentist confidence estimation. The estimation framework is not limited to Bayesian inference, and integrates with established approximate inference techniques such as variational inference and sampling based techniques, while maintaining provable coverage. The framework is universal in the sense that it does not require i.i.d.~data, any assumptions on the family of likelihood functions, and applies to any sequence of estimators. Further, using standard regret inequalities from online density estimation, the construction establishes a strong connection to classical maximum likelihood estimation.

\paragraph{Related Work} The core constructions in this paper go back to classical works, e.g.,~by \cite{wald1945sequential,darling1968some,robbins1970statistical}, and has seen revived interest in more recent literature \citep[e.g.,][]{wasserman2020universal,Emm23,flynn2024improved,lee2024unified}. More broadly, sequential mixing appears in online convex optimization, e.g., online aggregation, F-weighted portfolios \citep[c.f.,][]{orabona2019modern}, and prediction with expert advice \citep{cesa2006prediction}. Closely related is also the emerging literature on sequential testing using e-values \citep[see, e.g.,][]{grunwald2020safe} and PAC-Bayes bounds \citep[e.g.,][]{lugosi2023online,chugg2023unified}. Different to these works, here we focus on constructing confidence set for model parameters, with typical applications including inverse problems and compressed sensing.  We will refer to further related works in the relevant context, and defer a more complete historic account to \cref{sec:literature}.

\paragraph{Contributions} Our first contribution is to review classical and more recent results, establishing tight connections between seemingly disconnected works, and offering a unified framework for constructing confidence sequences using likelihood ratio martingales. The specific application of sequential mixing distributions to likelihood ratios, along with subsequent equivalence results and connections to Bayesian (approximate) inference, appears novel in this generality, to the best of our knowledge. We demonstrate the strength of the formulation by deriving improved confidence sets for sequential linear regression and for sparse estimation, with simplified proofs.

\section{Confidence Estimation via Sequential Likelihood Mixing}

We study the sequential confidence estimation problem in which the learner is given a set of model parameters $\Theta$, and a family of conditional densities $\cM = \{p_\theta(y|x): \theta \in \Theta\}$. 
The learner observes a data stream $x_1, y_1, x_2, y_2, \dots, x_t, y_t, \dots$, where $x_t$ is a covariate and $y_t \sim p_{\theta^*}(\cdot|x_t)$ is an observation sampled from a distribution with a ground-truth parameter $\theta^* \in \Theta$. We make no assumptions on how the covariates a generated, in particular, allow for an arbitrary dependence between $x_t$ and all prior data. The role of $x_t$ is secondary and is mostly introduced to match standard machine learning settings. For simplicity, we suppress the dependence on $x_t$ and denote $p_t(y|\theta) = p_{\theta}(y|x_t)$.
Our objective is to construct an anytime \emph{$(1-\delta)$-confidence sequence}: In each round $t \geq 1$, the learner outputs a set $C_t \subset \Theta$, such that $\theta^* \in C_t$ for all $t \geq 1$ holds with probability at least $1-\delta$.

\paragraph{Setting} Formally, we let $\cY$ be a measurable observation space and $\cX$ be a measurable covariate space. Let $\bP$ be the true data generating distribution over $(\cX \times \cY)^\infty$, with the data stream $({x_t, y_t})_{t=1}^\infty \sim \bP$. Unless stated otherwise, all probabilistic statements are with respect to this measure. The filtration corresponding to the observation sequence is $\cF_t = \{x_1,y_1, \dots ,x_{t-1},y_{t-1}, x_t\}$. We abbreviate the marginal distribution over the outcome $y_t$, conditioned on past data, with $\bP_t = \bP_{y_t}(\cdot | \cF_t)$. We assume that $\bP_t$ is dominated by a base measure $\xi$ over $\cY$ for all $t \geq 1$. 

We require that $\Theta$ is a measurable space and the density $p_t(y|\theta)$ defines a probability kernel $p_t(y|\theta) d \xi(y)$ from $\Theta$ to $\cY$. We consider both, finite and continuous $\Theta$, and let $\omega$ be a base measure over $\Theta$ (typically the counting measure for finite $\Theta$, or the Lebesgue measure for $\Theta \subset \bR^d$). The space of probability distributions over $\Theta$ is $\sP(\Theta)$. The Kullback-Leibner (KL) divergence between two distributions $\rho, \mu \in \sP(\Theta)$ with $\rho, \mu \ll \omega$ is $\KL(\rho\|\mu) = \int \rho(\theta) \log \frac{\rho(\theta)}{\mu(\theta)} d\omega(\theta)$. We make the following realizability assumption (which we will relax in \cref{sec:misspecified}),
\begin{assumption}[Realizability]\label{a:realizability}
There exists a $\theta^* \in \Theta$ such that $p_t(\cdot|\theta^*) = \frac{d\bP_t}{d\xi}$\,.
\end{assumption}
We remark that realizability and our definition of the model class imply that the distribution $y_t$ conditioned on $x_t$ is stationary and independent of past data, $\bP_t(\cdot|\cF_t) = \bP_t(\cdot|x_t) = p_{\theta^*}(\cdot|x_t)d\xi$. This requirement can be relaxed by defining a model class that is conditioned on the full history.

\paragraph{Confidence Sequences} Let $\delta \in (0,1)$ be a failure probability. Our a goal is to construct a $\cF_t$-adapted sequence of confidence sets $C_1, C_2, \dots \subset \Theta$, such that
\begin{align*}
\bP[\exists t \geq 1 \text{ s.t. }\theta^* \notin C_t] \leq \delta\,.
\end{align*}
In other words, the Type-I error over all time steps is at most $\delta$. We briefly remark that coverage can be achieved trivially, e.g., $C_t = \Theta$ for all $t \geq 1$ is a $(1-\delta)$-confidence sequence, and so is a randomly choosing $(1-\delta)$-fraction of $\Theta$, such that the probability $\{\theta \in C_t\}$ is $1-\delta$ for \emph{any} $\theta \in \Theta$. While these maintain coverage, they are not useful in practice because they do not control the Type-II error, or equivalently, the size of the confidence set, in a meaningful way. Hence, we aim to construct a $(1-\delta)$-confidence sequence and make $C_t$ as small as possible at the same time.

\subsection{Ville's Inequality and Sequential Likelihood Ratios}

We start by reviewing confidence sets based on sequential likelihood ratios \citep{robbins1970statistical,lai1985asymptotically}. Assume that in each round $t \geq 1$, after observing the data $\{x_1, y_1, \dots x_t, y_t\}$, we have a way to construct an estimator $\hat \theta_t \in \Theta$. Formally, the \emph{estimation sequence} $\hat \theta_1, \hat \theta_2, \dots$ is an $\cF_t$-adapted process in $\Theta$. In addition, we include a prior estimate $\hat \theta_0 \in \Theta$ that is chosen before any data is observed. Define the sequential likelihood ratio process for $\theta \in \Theta$,
\begin{align}
    R_t(\theta) = \prod_{s=1}^t \frac{p_s(y_s|\hat \theta_{s-1})}{p_s(y_s|\theta)} \,.
\end{align}
It is important to note that the likelihood of $y_s$ is evaluated under the estimate $\hat \theta_{s-1}$, making the predicted density $p_s(\cdot|\hat \theta_{s-1})$ independent of the observation $y_s$. It is straightforward to verify that $R_t(\theta^*)$ is a non-negative martingale with respect to the filtration $\cF_t$, i.e. 
\begin{align}
    \EE[R_t(\theta^*)|\cF_{t}] &= R_{t-1}(\theta^*) \int \frac{p_t(y|\hat \theta_{s-1})}{p_t(y|\theta)} d\bP_t(y)\nonumber \\
    & \stackrel{(*)}{=} R_{t-1}(\theta^*) \int \frac{p_t(y|\hat \theta_{s-1})}{p_t(y|\theta)} p_t(y|\theta) d\xi(y)  = R_{t-1}(\theta^*) \,.
\end{align}
The second equality $(*)$ uses realizability (\cref{a:realizability}).
The tail event that $R_t(\theta^*)$ growths large is bounded using Ville's inequality \citep{ville1939etude}, which generalizes Markov's inequality to non-negative supermartingales.
\begin{lemma}[Ville's Inequality \citep{ville1939etude}]
    Let \((M_t)_{t \geq 1}\) be a non-negative supermartingale. For any \(\alpha > 0\),
    \begin{align}
        \mathbb{P}\left[\sup_{t \geq 1} M_t \geq \alpha \right] \leq \frac{\mathbb{E}[M_1]}{\alpha} \,.
    \end{align}
\end{lemma}
Applying Ville's inequality to $R_t(\theta^*)$, noting that $R_t(\theta^*) \geq 0$ and $\EE[R_1] =1$, we get that
\begin{align*}
    \bP\left[\sup_{t \geq 1} R_t(\theta^*) \geq \frac{1}{\delta}\right] \leq \delta \,.
\end{align*}
By inverting the inequality, we obtain a $(1-\delta)$-confidence sequence $C_t = \{ \theta \in \Theta : R_t(\theta) \leq \frac{1}{\delta}\}$. As a last step, we re-write the inequality using the negative log-likelihood,
\begin{align*}
L_t(\theta) = L(\theta|x_1,y_1, \dots, x_t, y_t) = - \sum_{s=1}^t \log p_t(y_t|\theta) \,.
\end{align*}
We summarize the construction in the next theorem. Similar constructions have gained renewed interest more recently \citep[e.g.,][]{wasserman2020universal,Emm23}.
\begin{theorem}[Sequential Likelihood-Ratio \citep{robbins1970statistical}] \label{result:likelihood_ratio_confidence_set}
   Let $\hat \theta_0, \hat \theta_1, \hat \theta_2, \dots$ be an estimation sequence adapted to the filtration $\cF_t$. Define
    \begin{align*}
        C_t = \left\{ \theta \in \Theta : L_t(\theta) \leq \log \frac{1}{\delta} - \sum_{s=1}^t \log p_s(y_s|\hat \theta_{s-1}) \right\} \,.
    \end{align*}
    Then $C_t$ is a $(1-\delta)$-confidence sequence for $\theta^*$, i.e.~$\bP[ \forall t \geq 1, \theta^* \in C_t ] \geq 1-\delta$.
\end{theorem}
The confidence set defined in \cref{result:likelihood_ratio_confidence_set} does not require any assumptions on the family of likelihood functions other than realizability, and makes no assumptions on how the covariates $x_1, x_2, \dots$ are generated, or how the estimation sequence $\hat \theta_0, \hat \theta_1, \dots, \hat \theta_t$ is chosen. The size of the confidence set depends on the ability of the learner to produce an estimation sequence that maximizes the log-likelihood $\sum_{s=1}^t \log p_s(y_s|\hat \theta_{s-1})$. Since the estimated density $p_t(\cdot |\hat \theta_{s-1})$ is evaluated under the next observation $y_s$, it measures the accuracy of the learner to predict the next outcome.


\subsubsection{Maximum Likelihood Estimation} \label{sec:mle}
Maximum likelihood estimation is a cornerstone of classical frequentist statistics and closely related to sequential likelihood ratios. Formally, we define the maximum likelihood estimate (MLE)  as a minimizer of the negative log-likelihood,
\begin{align*}
    \hat \theta_t^{\MLE} = \argmin_{\theta \in \Theta} L_t(\theta) \,.
\end{align*}
Further, define the likehood ratio for $\nu, \theta \in \Theta$,
\begin{align} 
    R_t(\nu, \theta) = \prod_{s=1}^t \frac{p_s(y_s|\nu)}{p_s(y_s|\theta)}  \,.\label{eq:two-parameter-ratio}
\end{align}
It is tempting to use the maximum likelihood estimator in the construction of \cref{result:likelihood_ratio_confidence_set}. However, this fails because $R_t(\hat \theta_t^{\MLE}; \theta^*) = \max_{\nu \in \Theta} R_t(\nu; \theta^*)$ is not a super-martingale.

Fortunately, there are several ways to proceed, and we will come back to this in later sections. To set expectations, note that $R_t(\nu;\theta^*)$ is a martingale for any \emph{fixed} $\nu \in \Theta$. Ville's inequality implies that $\bP[\sup_{t \geq 1} R_t(\nu;\theta^*) \geq \frac{1}{\delta}] < \delta$. For finite $\Theta$, a union bound over $\Theta$ suffices to obtain a bound for all $\nu \in \Theta$ simultaneously, and, in particular, for any estimator $\hat \theta_t$. The resulting $(1-\delta)$-confidence sequence is
\begin{align}
C_t = \{ \theta : L_t(\theta) \leq \log \frac{1}{\delta} +  L_t(\hat \theta_t) + \log(|\Theta|) \} \,.
\end{align}
 Note that $C_t$ is a log-likelihood level set as in \cref{eq:level-set} with coefficient $\beta_t(\delta) = \log \frac{|\Theta|}{\delta} +  L_t(\hat \theta_t)$. The difference to \cref{result:likelihood_ratio_confidence_set} is that the union bound allows to compare directly to the log-likelihood $L_t(\hat \theta_t)$ for any estimator $\hat \theta_t$ computed in hindsight, using all observed data. The coefficient is minimized (resulting in the smallest confidence set) by the maximum likelihood estimator $\hat \theta_t^\MLE$. For continuous parameter sets, the argument extends using a suitable covering for $\Theta$.

\subsection{Prior Likelihood Mixing}

We now introduce a second approach to construct an anytime confidence sequence using the idea of `mixing'. Again, we make use of a martingale process, defined by integrating the likelihood ratio over a given prior distribution. An early reference to this idea is by \cite{robbins1970boundary}, who use mixture martingales to prove law-of-the-iterated-logarithm type bounds. Mixture martingales were popularized in the context of self-normalized bounds by \cite{pena2009self} as the `method of mixtures', with repeated interest in the machine learning community \citep[e.g.,][]{abbasi2011improved}.

Formally, let $\mu_0 \in \sP(\Theta)$ be a data-independent prior distribution. Recall the definition of $R_t(\nu;\theta)$ from \cref{eq:two-parameter-ratio}, and define the \emph{marginal likelihood ratio},
\begin{align*}
    Q_t(\theta) = \int R_t(\nu; \theta) d\mu_0(\nu) =  \frac{ \int \prod_{s=1}^t p_s(y_s|\nu)  d\mu_0(\nu)}{\prod_{s=1}^t p_s(y_s|\theta)} \,.
\end{align*}
Using Fubini's theorem, it follows that $Q_t(\theta^*)$ is a non-negative martingale with $\bE[Q_1(\theta^*)]=1$. Once again, we use Ville's inequality to bound the deviation,
\begin{align*}
    \bP\left[\sup_{t \geq 0} Q_t(\theta^*) \geq \frac{1}{\delta}\right] \leq \delta \,.
\end{align*}
The next theorem summarizes this result.
\begin{theorem}[Prior Likelihood Mixing]\label{result:prior_mixing}
    For any data-independent prior $\mu_0 \in \cP(\Theta)$, define
    \begin{align*}
        C_t &= \left\{ \theta \in \Theta: L_t(\theta) \leq  \log \frac{1}{\delta} - \log \int \prod_{s=1}^t p_s(y_s|\nu) d\mu_0(\nu) \right\} \,.
    \end{align*}
Then $C_t$ is a $(1-\delta)$-confidence sequence.
\end{theorem}
Note that the size of the confidence set is determined by the logarithm of the \emph{marginal likelihood} under the prior $\mu_0$, also referred to as Bayesian \emph{evidence}. To understand intuitively why the marginal likelihood offers a reasonable confidence threshold, assume that $\Theta \subset \bR^d$. A second-order Taylor expansion of $L_t(\theta)$ around the maximum likelihood estimate $\hat \theta_t^\MLE$ allows to compute the marginal likelihood in closed-form as a Gaussian integral. This is known as Laplace's method \citep{laplace1774memoire}. Under suitable technical regularity assumptions \citep[e.g.,][]{shun1995laplace}, the following asymptotic expansion holds,  
where $I_t(\theta) = \frac{\partial^2}{\partial \theta^2} L_t(\theta) \in \bR^{d\times d}$ is the empirical Fisher information matrix:
\begin{align*}
    - \log \int \exp(-L_t(\nu))  h(\nu) d\theta
    &\approx L_t(\hat \theta_t^{\MLE})  + \frac{1}{2}\log \det I_t(\hat \theta_t^\MLE) - \frac{d}{2} \log(2\pi) - \log \mu_0(\hat \theta_t^\MLE)  \,.
\end{align*}
This should be compared to the discussion in \cref{sec:mle}. Although $\max_{\nu} R_t(\nu,\theta^*)$ is not a martingale, the marginal likelihood ratio offers an approximation while preserving the martingale property. Further details are in \cref{sec:laplace}, and for an elementary introduction to Laplace's method and its relation to the method of mixtures, we refer to \citet[Chapter 20]{lattimore2020bandit}.






\subsection{Sequential Likelihood Mixing}\label{sec:mixing}
%
%
We now generalize the idea of mixing to the sequential setting.
Formally, let $\mu_1, \mu_2, \dots$ be an $\cF_t$-adapted sequence of \emph{mixing distributions} in $\sP(\Theta)$, and $\mu_0 \in \sP(\Theta)$ a `prior' mixing distribution that is chosen before any data is observed. We define the \emph{sequential marginal likelihood ratio}, 
\begin{align*}
    S_t(\theta)= \prod_{s=1}^t  \frac{\int p_s(y_s|\nu) d\mu_{s-1}(\nu)}{p_s(y_s|\theta)} \,.
\end{align*}
One can think of the mixing distribution $\mu_{s-1}$ as a weighted prediction, aggregation over parameters in $\Theta$. As before, Fubini's theorem is used to show that $S_t(\theta^*)$ is a non-negative martingale under $\cF_t$ with $\EE[S_0(\theta^*)] = 1$, and therefore Ville's inequality implies that
\begin{align*}
     \bP\left[\sup_{t \geq 0} S_t(\theta^*) \geq \frac{1}{\delta}\right] \leq \delta \,.
\end{align*}
Rewriting the concentration inequality using the negative log-likelihood yields the following result.
\begin{theorem}[Sequential Likelihood Mixing]\label{result:posterior_mixing} Let $\mu_0, \mu_1, \mu_2,\dots$ be a sequence of $\cF_t$-adapted mixing distributions in over $\Theta$. Define for $t \geq 1$,
    \begin{align*}
        C_t  = \left\{ \theta \in \Theta: L_t(\theta) \leq  \log \frac{1}{\delta} - \sum_{s=1}^t \log \int p_s(y_s|\nu) d\mu_{s-1}(\nu)\right\} \,.
    \end{align*}
    Then $C_t$ is a $(1-\delta)$-confidence sequence.
\end{theorem}
 \Cref{result:posterior_mixing} can be understood as a sequential analog of the method of mixtures (\cref{result:prior_mixing}), and recovers the sequential likelihood ratio confidence set (\cref{result:likelihood_ratio_confidence_set}) as a special case by setting $\mu_t$ to a Dirac measure on $\hat \theta_t$. The use of sequential mixing distributions has appeared in the literature before \cite[e.g.,][]{kirschner2018information,Emm23,flynn2024tighter}, but we are not aware of a reference that applies mixing directly to the sequential likelihood ratio.


\section{Applications, Connections and New Perspectives}

Sequential mixing provides a versatile framework for constructing confidence sequences for general parametric model classes via \cref{result:posterior_mixing}. A natural candidate for the mixing distribution is the Bayesian posterior, which we analyze in subsequent subsections. However, the range of possible mixing distributions extends far beyond Bayesian inference, often offering computational and statistical trade-offs. While there is no universally optimal choice, a key consideration is that the resulting confidence set is typically tighter when the mixing distribution $\mu_t$ is concentrated around the true parameter $\theta^*$. For completeness, we list several possible choices below, noting that all of these, including point estimates and sampling-based approximations (e.g., Monte Carlo methods), yield valid $(1-\delta)$-confidence sequences.

\begin{itemize}
    \item \textbf{Bayesian Inference:} A natural approach is updating the mixing distributions via Bayes' rule: $\mu_t(\theta) \propto \prod_{s=1}^t p_t(y_t|\theta) \mu_0(\theta)$. This is discussed in detail in \cref{sec:bayes}.
    \item \textbf{Variational Inference:} When the marginal likelihood is intractable, variational inference methods provide computational lower bounds on the evidence, see \cref{sec:elbo} below.
    \item \textbf{Posterior Sampling:} Using sampling-based approximations, such as Langevin dynamics \citep{dwivedi2019log}, we can set $\mu_t(\theta) = \frac{1}{N} \sum_{i=1}^N \delta_{\tilde{\theta}_t^i}$ for $N$ samples $\tilde \theta_t^1, \dots, \tilde \theta_t^N \in \Theta$.
    \item \textbf{Online Estimation:} Complexity bounds from online density estimation allow to relate the sequential mixing confidence set to classical supervised estimation, see \cref{sec:oco}.
\end{itemize}

\subsection{From Bayesian to Frequentist Inference}\label{sec:bayes}

A natural choice for the mixing distribution is the Bayesian posterior, which establishes a fundamental connection between frequentist confidence estimation and Bayesian inference. To explore this relationship, we first formally define the Bayesian inference model.
\begin{assumption}[Bayesian Inference]\label{a:bayes}
    In the Bayesian inference model, the learner defines a prior distribution $\mu_0 \in \sP(\Theta)$ over model parameters (independent of the data), and predicts using the posterior distribution $\mu_t(\theta) \propto \prod_{s=1}^{t-1} p_s(y_s|\theta) \mu_0(\theta)$.
\end{assumption}
The main result of this section establishes that if the mixing distributions are computed according to Bayes' rule, then prior likelihood mixing (\cref{result:prior_mixing}) and sequential likelihood mixing (\cref{result:posterior_mixing}) are equivalent. A further application of Bayes rule shows that any (realizable) Bayesian model can be turned into a $(1-\delta)$-confidence sequence by comparing the log posterior probability $\log \mu_t(\theta)$ to the log prior probability $\log \mu_0(\theta)$. This is known as \emph{prior-posterior ratio confidence set} \citep{waudby2020confidence}: 
\begin{align*}
    C_t =  \left\{ \theta \in \Theta: - \log \mu_t(\theta) \leq  \log \frac{1}{\delta} - \log \mu_0(\theta) \right\} \,.
\end{align*}
The equivalence result is foreshadowed in the works by \citet{waudby2020confidence} and \citet{neiswanger2021uncertainty}, who establish the posterior-ratio confidence set and the connection to the marginal likelihood. The explicit equivalence to the sequential mixing framework, however seems to be absent in prior works, and is formally given in \cref{result:mixing-equivalence} below. 

\begin{theorem}[Mixing Equivalence]\label{result:mixing-equivalence} If the mixing distributions are chosen according to Bayes' rule, prior likelihood mixing (\cref{result:prior_mixing}) and sequential mixing (\cref{result:posterior_mixing}) are equivalent.
\end{theorem}
\begin{proof}
   The result follows by applying Bayes' rule recursively to show the following equality, $\sum_{s=1}^t \log \int p_s(y_s|\nu) d\mu_{s-1}(\nu) = \log \int \prod_{s=1}^t p_s(y_s|\nu) d\mu_{0}(\nu)$.
\end{proof}

The surprising consequence is, that within the Bayesian inference model, sequential mixing provides no statistical advantage compared to averaging the likelihood over the prior. Less surprisingly though, Bayes' rule can be understood as an incremental update rule to compute the marginal likelihood. In this sense, the equivalence can be re-stated as recovering prior mixing (\cref{result:prior_mixing}) as a special case of sequential mixing (\cref{result:posterior_mixing}). However, note that for mixing distributions outside the Bayesian model, the equivalence does not hold in general, leaving the possibility to find non-Bayesian mixing distributions that achieve faster convergence. We come back to this idea in \cref{sec:oco}.

Next, we state a second implication of Bayes' rule, the prior-posterior ratio confidence set. 
\begin{lemma}[Prior-Posterior Ratio Confidence Set \citep{waudby2020confidence}] \label{lem:posterior_ratio_confidence_set}\\
    For any realizable Bayesian model, the following defines a $(1-\delta)$-confidence sequence:
\begin{align*}
    C_t &=  \left\{ \theta \in \Theta: - \log \mu_t(\theta) \leq  \log \frac{1}{\delta} - \log \mu_0(\theta) \right\} \,.
\end{align*}
Moreover, the confidence set is equivalent to the construction in \cref{result:prior_mixing,result:posterior_mixing}.
\end{lemma}
\begin{proof}
    Note that $\log \mu_t(\theta) = \log \mu_0(\theta)  + L_t(\theta) - \log \int \prod_{s=1}^t p_s(y_s|\nu) d\mu_{0}(\nu)$ holds for all $\theta \in \Theta$ by Bayes' rule. Substituting the equality into \cref{result:prior_mixing} gives the result.
\end{proof}
The remarkable conclusion is that any realizable Bayesian model can be turned into a frequentist confidence set by thresholding the log posterior probability relative to the log prior probability. As a caveat, it is tempting to think of $C_t$ as a Bayesian credible region, however, the posterior credible probability $\mu_{t-1}(C_t)$ is typically not $1-\delta$. Further, the confidence set, by construction, never rejects parameters in the null set of the prior distribution, unlike in classical Bayesian inference. In any case, a sensible choice is $\Theta = \supp(\mu_0)$, as long as the realizability condition (\cref{a:realizability}) is satisfied, that is, $\theta^* \in \Theta$ defines the true likelihood of the data. For an application of the prior-posterior confidence set to sequential sampling without replacement, we refer to \citet{waudby2020confidence}.

As a consequence of the prior-posterior ratio confidence set and the mixing equivalence, the confidence sets for sequential linear regression by \citet{neiswanger2021uncertainty,flynn2024improved,flynn2024tighter} and earlier work by \cite{abbasi2011improved} are essentially equivalent, as we demonstrate below. Moreover, a lower bound by \citet{lattimore2020bandit} shows that the construction is tight without further assumptions on the data generation distribution. 

\paragraph{Sequential Linear Regression} 
To illustrate the utility of the Bayesian perspective, we consider sequential linear regression with a Gaussian prior and likelihood. To preempt any concerns, we remark that the Gaussian assumption can be relaxed to sub-Gaussian distributions, as we explain in \cref{sec:subgaussian}. Formally, let $\theta^* \in \Theta = \bR^d$, with multivariate Gaussian prior $\cN(\theta_0, V_0^{-1})$ centered at $\theta_0 \in \bR^d$ and prior precision matrix $V_0 \in \bR^{d \times d}$, where commonly $V_0 = \lambda \eye_d \in \bR^{d\times d}$ for a regularizer $\lambda > 0$. The observation likelihood is Gaussian,  $y_t \sim \cN(x_t^\top\theta^*, \sigma^2)$ for a feature vector $x_t \in \bR^d$ and known observation variance $\sigma^2 > 0$. The Gaussian posterior is $\mu_t = \cN(\hat \theta_t^\RLS, V_t^{-1})$, where $\hat \theta_t^\RLS$ is the regularized least squares (RLS) estimate,
\begin{align*}
\hat \theta_t^\RLS = \argmin_{\theta \in \bR^d} \frac{1}{2 \sigma^2} \sum_{s=1}^t \big(\ip{x_s, \theta} - y_s\big)^2 + \frac{1}{2} \|\theta - \theta_0\|_{V_0}^2\,.
\end{align*}
Here, $V_t = \frac{1}{\sigma^2}\sum_{s=1}^t x_s x_s^\top + V_0$ is the posterior precision matrix, and we use the notation $\|\nu\|_A^2 = \nu^\top A \nu$ for $\nu \in \bR^d$ and $A \in \bR^{d\times d}$. The prior and posterior densities are explicitly given as follows:
\begin{align*}
    \mu_0(\theta) &= (2 \pi)^{-2/k} (\det V_0)^{1/2} \exp\big(- \tfrac{1}{2}\|\theta - \theta_0\|_{V_0}^2 \big) \\
    \mu_t(\theta) &= (2 \pi)^{-2/k} (\det V_t)^{1/2} \exp\big(- \tfrac{1}{2}\|\theta - \hat \theta_t^\RLS\|_{V_t}^2 \big)
\end{align*}
Applying \cref{lem:posterior_ratio_confidence_set} with the Gaussian posterior, we get the following $(1-\delta)$-confidence sequence:
\begin{align*}
    C_t^\RLS = \left\{ \theta \in \bR^d : \frac{1}{2}\|\theta - \hat \theta_t^\RLS\|_{V_t}^2 \leq \log \frac{1}{\delta} + \frac{1}{2}\log \det V_t - \frac{1}{2}\log \det V_0 + \frac{1}{2}\|\theta  - \theta_0\|_{V_0}^2 \right\}\,.
\end{align*}
An important feature of the bound is that it scales with the \emph{effective dimension} or \emph{total information gain} $\gamma_t = \frac{1}{2}\log \det V_t - \frac{1}{2}\log \det V_0$ of the data \citep[c.f.~][]{huang2021short}, which can be much smaller than the parameter dimension $d$ \citep{srinivas2009gaussian}. 
Note also that the confidence set does \emph{not} require a known bound on the norm $\|\theta^*\|_2 \leq S$, which is required in all prior work that we are aware of. If such a bound is available, a direct approach is to define the Gaussian prior and posterior directly over the restricted set $\cB_S = \{\theta \in \bR^d : \|\theta\|^2 \leq S\}$. However, in this case, the normalization constant is not easily computed in closed form. Instead, we intersect $C_t^\RLS$ with the norm ball $\cB_S$. Relaxing the confidence set further, and choosing $V_0 = \lambda \eye_d$ and $\theta_0 = 0$, we eventually arrive at
\begin{align*}
    C_t^\RLS \cap \cB_S \subset \left\{ \theta \in \bR^d : \frac{1}{2} \|\theta - \hat \theta_t^\RLS\|_{V_t}^2 \leq \log \frac{1}{\delta} + \frac{1}{2}\log \det V_t - \frac{d}{2}\log \lambda+ \frac{\lambda}{2}S^2 \right\} \,.
\end{align*}
The last line essentially recovers the influential result by \citet{abbasi2011improved}, albeit avoiding a lower-order cross-term, improving the bound by up to a factor of two. 
The proof of \citet{abbasi2011improved} uses the method of mixtures, but mixing the noise martingale $S_t = \sum_{s=1}^t \epsilon_s x_t$ over a centered prior, instead of directly mixing the likelihood ratio. 
More recent work by \cite{flynn2024improved} achieves the tighter result using a similar sequential mixing idea, however, the likelihood framework and connection to Bayesian inference is not mentioned. A direct extension is to heteroscedastic noise, $y_t \sim \cN(x_t^\top\theta^*, \sigma_t^2)$ with known variance $\sigma_t^2$ \citep[c.f.,][]{kirschner2018information}. Another, more involved extension is to unknown mean and variance \cite[c.f.,][]{chowdhury2023bregman}. \looseness=-1

\paragraph{Gaussian Process Regression}
We remark that the confidence set for sequential linear regression can be equivalently stated for non-parametric Gaussian processes regression in infinite-dimensional kernel Hilbert spaces (RKHS) using the `kernel trick'. Our derivation improves (up to a factor of two) the results by \cite{abbasi2012thesis,chowdhury2017kernelized,whitehouse2023sublinear} and recovers more recent results by \cite{neiswanger2021uncertainty,flynn2024tighter}, who do not state the equivalence.



\subsection{Variational Confidence Sets}\label{sec:elbo}
While the confidence set construction and its relation to Bayesian inference is universal and holds for any prior and family of likelihood functions, the price to pay is that computing the marginal likelihood, or equivalently, the posterior distribution, is intractable in general. Fortunately, approximate inference methods have been well studied in the field of Bayesian inference. 
Our starting point is a variational inequality for the marginal likelihood $\int \prod_{s=1}^t p_s(y_s|\theta)d\mu_0(\theta) = \int \exp(- L_t(\theta)) d\mu_0(\theta)$ and the Kullback-Leibner (KL) divergence, often attributed to \citet{donsker1983asymptotic}. 
\begin{lemma}[Variational Inequality]\label{lemma:variational-kl}
For any two distributions $\mu,\rho \in \sP(\Theta)$ and any measurable function $g : \Theta \rightarrow \bR$,
    \begin{align*}
    \log \int \exp(g(\theta)) d\mu \geq \int g(\theta) d\rho(\theta) - \KL(\rho\|\mu) \,.
    \end{align*}
    Moreover, the inequality becomes an equality for $d\rho(\theta) \propto \exp(g(\theta)) d\mu(\theta)$.
\end{lemma}
The inequality plays a central role in variational inference, and is typically stated as the \emph{evidence lower bound} (ELBO), by specializing \cref{lemma:variational-kl} using $g(\theta) = - L_t (\theta)$ and $\mu = \mu_0$,
\begin{align*}
 \log \int \exp(- L_t(\theta)) d\mu_0 \geq \ELBO_t(\rho) := -\int L_t(\theta) d\rho(\theta) - \KL(\rho\|\mu_0) \,.
\end{align*}
For the given choices, the inequality becomes tight when $\rho = \mu_t$ is the Bayesian posterior.
Variational inference aims at numerically maximizing the evidence lower bound over a parametric family of posterior distributions \citep{jordan1999introduction}, see also \citep{blei2017variational}. In the context of confidence estimation, the key insight is that the variational inequality allows to relax the marginal likelihood that defines the confidence sequence in \cref{result:prior_mixing}. This result has been recently stated by \citet{lee2024improved} in the specialized context of logistic and multinomial bandits, and similar bounds are well-known in the PAC-Bayes literature \citep[e.g.,][]{zhang2006varepsilon,chen2022unified,alquier2024user}. Here, we emphasize the connection to variational inference and the evidence lower bound as a way to define a confidence coefficient with valid anytime coverage.
\begin{theorem}[Evidence Lower Bound Confidence Set]\label{result:elbo_confidence_set}
    For any $\cF_t$-adapted sequence of distributions $\mu_t \in \sP(\Theta)$ and a data-independent prior $\mu_0 \in \sP(\Theta)$, define\looseness=-1
    \begin{align*}
		 C_t = \left\{\theta \in \Theta : L_t(\theta) \leq  \log \frac{1}{\delta} - \ELBO_t(\mu_t)  \right\} \,.
	\end{align*}
    Then $C_t$ defines a $(1-\delta)$-confidence sequence. Moreover, if $\rho_t$ is chosen as the Bayesian posterior, the result is equivalent to the marginal likelihood confidence set in \cref{result:prior_mixing}.
\end{theorem}
The practical implication of this result is that it provides a tool to trade off computational tractability and statistical efficiency. In particular, standard variational inference methods can be converted into a confidence set with provable coverage, simply by thresholding the negative log-likelihood by the attained evidence lower bound. Another possibility is to make ad-hoc choices for the posterior to obtain closed-form expressions for confidence sets, e.g.~for logistic regression \citep{lee2024unified}.

\subsection{Oracle Complexity Bounds via Online Estimation}\label{sec:oco}

The size of the sequential mixing confidence set in \cref{result:posterior_mixing} depends on the ability of the learner to produce a sequence of mixing distributions that minimize the cumulative log loss,
\begin{align*}
   \sum_{s-1}^t l_s(\mu_{s-1}) =  - \sum_{s=1}^t \log \int p_s(y_s|\nu) d\mu_{s-1}(\nu) \,.
\end{align*}
The field of \emph{online density estimation} studies algorithms for minimizing the cumulative log loss, and offers a rich literature on complexity bounds for the regret \citep[e.g.,][]{vovk1990aggregating,vovk1997competitive,zhang2006varepsilon,rakhlin2014online}. Here, regret is defined as the difference between the cumulative log loss, and the best prediction in hindsight, which, in the simplest case, coincides with the maximum likelihood estimate. Specifically, we assume that the mixing distributions $\mu_1, \mu_2, \dots$ are chosen by an \emph{online estimation} algorithm, in a way that ensures a bound $B_t \geq 0$ on the log-regret,
\begin{align*}
    \Lambda_t = \sum_{s=1}^t l_s(\mu_{s-1}) - \min_{\theta} L_t(\theta) \leq B_t \,.
\end{align*}
For a complete introduction, we refer the reader to the standard literature \citep[e.g.,][]{cesa2006prediction,orabona2019modern,shalev2012online}. 

The next result demonstrates how the sequential likelihood mixing framework relates to maximum likelihood estimation, using regret inequalities from online estimation. 
\begin{theorem}[Regret-To-Confidence]\label{result:regret}
    Assume there exists an online estimation algorithm such that the log-regret is bounded almost-surely, $\Lambda_t \leq B_t$, for a predictable sequence $B_t \geq 0$. Define 
    \begin{align*}
    C_t = \left\{ \theta \in \Theta: L_t(\theta) \leq  \log \frac{1}{\delta} + L_t(\hat \theta_t^\MLE) + B_t \right\} \,.
    \end{align*}
    Then $C_t$ defines a $(1-\delta)$-confidence sequence.
\end{theorem}
\begin{proof}
The result follows directly from \cref{result:posterior_mixing}, by introducing $L_t(\hat \theta_t^\MLE)$ and using the definition of the regret.
\end{proof}
The importance of the result is that the \emph{existence} of an online estimation algorithm with a \emph{known} regret bound, allows to define a valid $(1-\delta)$-confidence sequence, relative to the log-likelihood of the MLE and the complexity bound for online estimation. In particular, the construction does \emph{not} require access to the predictions of the online learning algorithm.  Moreover, the confidence coefficient can be computed using standard supervised learning algorithms, eliminating the need for computing the marginal likelihood. Lastly, the complexity term offers a more interpretable bound on the confidence coefficient. We will come back to several concrete examples below.

The use of regret inequalities to derive concentration inequalities goes back to at least \cite{dekel2010robust,abbasi2012online}. The ``Online-to-Confidence-Set'' conversion by \citet{abbasi2012online} is, however, different in a subtle way, and requires access to the predictions of the online learning algorithm. We recover (and improve upon) their result using a generalization of \cref{result:regret} in \cref{sec:sparse}. Later works extend this idea, for example, to (multinomial) logistic bandits \citep{lee2024improved}. A similarly flavoured result is by \cite{abeles2024generalization} in the context of PAC-Bayes generalization bounds. Worth mentioning is also the work by \cite{rakhlin2017equivalence}, who prove an equivalence between regret bounds and tail inequalities.

\paragraph{Logistic Regression} We illustrate \cref{result:regret} in the sequential logistic regression setting. Let $\Theta = \{\theta \in \bR^d : \|\theta\|_2 \leq S\}$ for a norm bound $S > 0$. The likelihood is a Bernoulli distribution, $p_t(y|\theta) = \Ber(\phi(\ip{\theta, x_t}))$ with the logistic link function $\phi(z) = (1 + e^{-z})^{-1}$  and the covariates $x_t \in \bR^d$. Following along the lines of \citet{lee2024improved}, we invoke a regret bound for logistic regression by \cite{foster2018logistic}, who prove an online learning algorithm that achieves $\Lambda_t \leq 10d \log \left( e + \frac{St}{2d} \right)$. \cref{result:regret} immediately implies the following $(1-\delta)$-confidence sequence:
\begin{align*}
    C_t = \left\{ \theta \in \Theta: L_t(\theta) \leq  \log \frac{1}{\delta} + L_t(\hat \theta_t^\MLE) + 10d \log \left( e + \frac{St}{2d}  \right)\right\} \,.
\end{align*}
We remark that the confidence coefficient above improves upon the result of \citet[Theorem 1]{lee2024improved}, as a consequence of directly applying Ville's inequality to the likelihood ratio.

\paragraph{Compressed Sensing} We come back to the sequential linear regression setting described below \cref{lem:posterior_ratio_confidence_set}, with the additional assumption that the true parameter $\theta^*$ is $k$-sparse, i.e.~$\|\theta^*\|_0 \leq k$. One way to account for the sparsity assumption is to set $\Theta_k = \{\theta \in \bR^d : \|\theta\|_0 < k, \|\theta\|_2 \leq S\}$. For sparse linear regression, \citep{gerchinovitz2011sparsity} proposes online learning algorithm  that achieves $\Lambda_t \leq C_0 k \log(t)$ for a constant $C_0 > 0$. \cref{result:regret} implies the following a confidence sequence,
\begin{align*}
    C_t = \left\{ \theta \in \Theta: L_t(\theta) \leq  \log \frac{1}{\delta} + L_t(\hat \theta_t^\MLE) + C_0 k \log(t)\right\} \,.
\end{align*}
Note that here MLE is defined over the sparse set $\Theta_k$, which poses computational challenges. On the other hand, the confidence set by \citet{abbasi2012online} requires access to the predictions of the online learning algorithm. The construction therefore suffers a similar fate, as finding a computationally efficient algorithm for sparse linear prediction is still an open problem.


\paragraph{Finite Model Identification} Assume that the parameter set $\Theta$ is finite, i.e.~$|\Theta|< \infty$. The key insight is to note that the log-loss is $\eta$-exp-concave for $\eta \leq 1$, i.e.~$\exp(-\eta l_t(\mu))$ is concave in $\mu \in \sP(\Theta)$ for all $\eta \leq 1$ (in fact, it is linear for $\eta=1$). This highlights the importance of using mixing distributions, as, in general, $-\log p_t(y_t|\theta)$ is \emph{not} exp-concave as a function of $\theta$.

The standard approach for online learning with exp-concave functions is the \emph{exponential weights algorithm} \citep[EWA, ][]{littlestone1994weighted,freund1997decision}, see also the book by \citet{cesa2006prediction}. For $\eta=1$, EWA is equivalent to Bayesian inference and the prediction $\mu_t$ is equal to the Bayesian posterior (see \cref{alg:cew}). With a uniform prior, the regret of EWA satisfies $\Lambda_t \leq \log(|\Theta|)$, uniformly over all data sequences. The proof is provided for completeness in \cref{sec:cew}. Using \cref{result:regret}, we obtain the following $(1-\delta)$-confidence sequence\looseness=-1
\begin{align*}
    C_t = \{\theta \in \Theta : L_t(\theta) \leq \log \frac{1}{\delta} + L_t(\hat \theta_t^\MLE) + \log |\Theta| \} \,.
\end{align*}
The result should be compared to the standard union bound argument (\cref{sec:mle}). While the bound is the same, note that we obtained \cref{result:regret} as a \emph{relaxation} of \cref{result:posterior_mixing}. Hence, the sequential mixing confidence set (with the Bayesian posterior as mixing distributions) is never worse, and possibly tighter for benign data and structured model classes \citep[e.g.,][]{auer2002adaptive,cesa2007improved,de2014follow}. The confidence set can also be written for the maximum a-posteriori estimate (MAP), in which case the role of the prior distribution becomes more apparent:
\begin{align*}
    C_t = \left\{\theta \in \Theta : L_t(\theta) \leq \log \frac{1}{\delta} + \min_{\nu \in \Theta} \big( L_t(\nu) - \log \mu_0(\nu) \big)\right\} \,.
\end{align*}
Lastly, we remark that the exponential weights algorithm can be generalized to the continuous setting. The confidence sequence derived from the regret bound of continuous exponential weights is equivalent to the ELBO confidence set in \cref{result:elbo_confidence_set}. This is another consequence of the mixing equivalence. We refer to \cref{sec:cew} for further details.

\section{Conclusion}

We presented a unifying framework for constructing confidence sequences using sequential likelihood mixing, with deep connections to Bayesian inference, maximum likelihood estimation and regret inequalities from online convex optimization. Many of the results presented here have appeared in some form in the literature before, although scattered and seemingly disconnected, or in specialized settings. An inevitably incomplete overview of related works is given in \cref{sec:literature}.\looseness=-1 


Of course, the story does not end here. Importantly, the realizability assumption can be relaxed, for example, to sub-Gaussian families (Appendix \ref{sec:subgaussian}) and convex model classes (Appendix \ref{sec:convex}). The sequential likelihood ratio is not the only martingale that can be turned into a confidence sequence. Another natural extension to \emph{tempered likelihood ratios} is discussed in \cref{sec:tempered}, which has been suggested as a way to make the Bayesian model update more robust \citep[e.g.,][]{grunwald2012safe}. The essentially equivalent setting of sequential testing, and the literature on e-values is another staring point for further investigation \citep{grunwald2020safe}.

One of the main objection against the Bayesian approach is often that it is generally subjective, as it depends on the choice of prior. As we demonstrate here, this does not prevent us from constructing frequentist confidence sets, as long as the prior is chosen independently of the data. Although the confidence sets depend explicitly on the Bayesian posterior, the random deviations are controlled in the frequentist sense, with respect to the true data distribution. Nevertheless, it is true that the size of the confidence set depends on how well the prior is concentrated on the true data distribution, but this is analogous to specifying structural model assumptions in a purely frequentist setup. \looseness=-1

In a similar spirit, \citet{wasserman2020universal} briefly mentions the prior likelihood mixing for constructing a confidence set, noting that it ``requires specifying a prior and doing an integral''. We certainly do not deny the challenges associated with computing the marginal likelihood, but would argue that the marginal likelihood appears as a natural quantity in the construction of confidence sequences, and establishes a fundamental connection to Bayesian inference. Approximate inference techniques (e.g., variational inference and sampling-based approximations) make this idea viable, maintaining provable coverage guarantees. Moreover, in settings where prior data is available, a promising direction is to learn structured priors from the data.

\acks{We thank Guillaume Obozinski for helpful discussions on the topic.}

\bibliography{bibliography}

\newpage
\appendix
\crefalias{section}{appendix} 

\section{Related Literature}\label{sec:literature}
The likelihood principle is at the core of modern statistics. It is only natural that it features in tasks associated with modern statistical decision theory such as the likelihood ratio test. Its history dates back to \cite{wald1945sequential} who considered sequential likelihood ratios for hypothesis testing rather than constructing confidence sequences. One of the earliest references that explicitly considers confidence sets constructed from likelihood ratios is by \citet{darling1968some,robbins1970statistical}. The difference in treatment persists in literature till today. However, the difference is mostly superficial. There is a tight link between confidence sets (or sequences thereof) and statistical tests (and sequential execution of those). Given any of a confidence set, we can directly construct a test, and any family of point-wise tests can be turned into a confidence set. The literature is split between analyzing sequential confidence sets and sequential testing. Yet, more splits occurs when we look at different fields.  A large portion of contributions are made within the field of mathematical statistics, which does not always overlap with modern machine learning theory and sequential decision making. Below we will try to provide links to works from both branches in the literature, and explain differences and similarities, giving the hint of historical, but not yet exhaustive, timeline. 


\paragraph{Sequential Likelihood Ratios: Origins}
While ideas of sequential likelihood ratios were explored earlier, the work by \citet{robbins1970statistical} is the foundational pillar of later works. This work already contains the ideas of combining Ville's inequality and likelihood ratios to obtain anytime-valid confidence sequences. Mixing with a prior distribution is also presented, which stems from \citet{robbins1970boundary}, with example calculations for example densities, including Gaussian, along with the derivation of the law-of-iterated-logarithm. The work lacks explicit extensions to multi-parameter probability distributions, (un-controlled) covariates. \citet{robbins1970statistical} does not discuss relation to marginal likelihood and Bayesian inference. Later, \cite{Robbins1972} shows how to construct testing scheme with increasing thresholds using these martingale techniques, and proves asymptotic power of these tests. \citet{lai1976confidence} builds upon these works and considers the likelihood ratio for exponential families explicitly. This is later used in order to derive asymptotically optimal bandit algorithm in \cite{lai1985asymptotically}. This is by no mean exhaustive list, however to the best our knowledge this line of work always concerns itself with single parameter families with the notable exception \cite{lai1994modification}. The reviews of \cite{lai2001sequential, lai2009martingales} summarizes the above results from the past decades. Basic martingale tools such such as method of mixtures are extensively used. Closely related to the current work is also the work by \cite{wasserman2020universal} on \emph{universal inference}, which uses the sequential likelihood ratio tests to construct confidence sets, building on the ideas of \cite{robbins1970statistical}. The work heavily focuses on the i.i.d.~setting and the split likelihood ratio test, but refers also a sequential treatment in later sections. 

\paragraph{Mathematical Statistics: Sequential Testing, Mixing and E-Values}
That the likelihood ratio is a martingale under the true distribution as has been clearly demonstrated in the above works. However, other random processes can act as testing martingales, and are referred to in this context as \emph{e-values}. For example, testing sequences can be constructed for classes of likelihood functions (e.g., sub-Gaussian, symmetric, etc). In this context, the method of mixtures and self-normalized processes are extensively used in the seminal work by \citet{pena2009self}. These ideas were further generalized by \cite{Howard2018,Howard2020,chowdhury2023bregman} for many other sub-families, different spaces of random variables beyond Hilbert spaces. The modern treatment of safe anytime inference and game-theoretic interpretation using e-values was initiated in a sequence of works by \citet{vovk2014game,grunwald2020safe,wang2022false,shafer2019game}. The very recent book by \citet{ramdas2024hypothesis} provides an excellent overview, and includes more complete historic account and recent developments.

\paragraph{Online Estimation, Regret Bounds and Universal Portfolios} There is a rich history on works that study the problem of minimizing the log loss. Regret bounds for online aggregation and online density estimation were pioneered in early works by \citet{vovk1990aggregating,littlestone1994weighted}. A standard reference is the book by \citet{cesa2006prediction}.  In mathematical finance, minimizing the log-loss is equivalently interpreted as maximizing the log return of an allocation strategy, and mixing algorithms are already used by \citet{cover1984algorithm}, most famously for the \emph{universal portfolio}  algorithm \citep{cover1991universal}. The use of regret bounds for constructing concentration inequalities however seems to appear much later in works by \citet{rakhlin2017equivalence,orabona2023tight,ryu2024confidence}. For an introductory text, we refer to \citet{orabona2019modern}. A notable result by  \citet{rakhlin2017equivalence} establishes an equivalence between regret bounds and martingale concentration inequalities.

The idea of using regret bounds for constructing confidence sequences for parametric inverse problems, similar to \cref{result:regret}, is by \citet{abbasi2012online}, who use regret bounds for deriving confidence sets sparse linear regression in the context of linear bandits. We re-derive their result in the mixing framework, with improved constants (see \cref{sec:sparse}). These ideas were further extended to generalized linear models, albeit without the use of mixing in by \cite{Jun2017,lee2024improved}. Lastly, \cite{Emm23} builds on this with likelihood ratios, and interpret the sequential likelihood ratio in \cref{result:likelihood_ratio_confidence_set} as an online prediction game, and analyze the regret of follow-the-regularized leader algorithm which coincides with the running MLE. Different to this work, they analyze it in terms of the parameters of the distribution instead of using mixing distributions, as done in this work. 


\paragraph{Sequential Decision Making} Anytime confidence sets play a central role in sequential decision making, including bandit algorithms \citep{lattimore2020bandit} and reinforcement learning theory \citep{agarwal2019reinforcement,foster2023foundations}. There is a large variety of settings and assumptions, but shared is that the learning algorithms controls, to varying degree, the evolution of covariates (e.g., evolving a state by choosing actions). The consequence is that the observation at step $t$, in general, depends on the previous observation history. The data becomes inherently non-i.d.d., and tools such as those described in this work, are essential. Uncertainty estimates or often explicitly used to balance exploration and a goal-driven objective, or to account for safety constraints. Confidence intervals are famously used in the family of upper confidence bound algorithms for action selection \citep{lai1985asymptotically,lai1987adaptive,auer2002finite}. In this context, \citet{abbasi2011improved} uses the method of mixtures to construct improved bounds for the linear bandit. There is a large body of subsequent works, and we can only list a few examples \citep[e.g.,][]{chowdhury2017kernelized,kirschner2018information,Faury2020,Mutny2021a,flynn2024improved,flynn2024tighter}. \citet{lee2024unified} uses the variational inequality as in \cref{lemma:variational-kl} to derive confidence sets for logistic regression. By using the regularity of the likelihood, namely smoothness, and boundedness of the parameter set, they derive tight bounds on the log-likelihood relative to the maximum likelihood estimate. Closely related is the work by \citet{Emm23}, who study the sequential likelihood ratio test of \citet{robbins1970statistical} with an explicit application to sequential decision making.



\paragraph{Bayesian Inference and Frequentist Statistics} Links between Bayesian inference and (frequentist) worst-case bounds are well-known and appear, more or less explicitly \citep[e.g.,][]{zhang2006varepsilon}, although, we believe, not in the generality for constructing confidence sequences as presented here. Arguably, the tight relationship is not as well-known as perhaps it should be, and it provides a powerful tool to derive anytime-valid confidence sequences in general parametric settings. The prior-posterior ratio confidence set is by \citet{waudby2020confidence}. Surprisingly, we could not find earlier references, and the connection was clearly not apparent in earlier works \citep[e.g.,][]{abbasi2011improved}. The prior-posterior confidence set is applied to Gaussian process regression by \citet{neiswanger2021uncertainty}, however, the work does again not mention the (near-)equivalence to earlier results (including, for example, closed-form expressions for confidence intervals in the kernelized setting).


\paragraph{Generalization Bounds: PAC-Bayes and Conformal Prediction}
Generalization and shrinking confidence sets are tightly related. Given a confidence set, the image of the confidence set under a forward operator yields bounds on the generalization error. However, this construction is not always practical, and can lead to suboptimal bounds that scale with the model class complexity. A more direct analysis of the generalization error, possibly with additional assumptions, often leads to tighter bounds. Generalization bounds are the main objective in the field of PAC-Bayes learning \citep[e.g.,][]{alquier2024user}. Not surprisingly, PAC-Bayes bounds, regret inequalities and variational bounds are tightly connected as well \citep{lugosi2022generalization,haddouche2022online,chugg2023unified}. Notably, \citet{lugosi2023online} derive generalization bounds that rely purely on the \emph{existence} of a regret inequality, similar to \cref{result:regret}. The emphasize that the main difference between this line of work and the current work is to goal of obtaining generalization bounds, opposed to confidence sets on the model parameters. The difference manifests in the loss function used (in particular, we use the log-loss which is exp-concave and allows for fast rates). Yet another line of work is on \emph{conformal prediction} \citep{angelopoulos2021gentle,angelopoulos2024theoretical}, which again targets the generalization error in a distribution free setting.

\section{Laplace's method}\label{sec:laplace}
Recall from section \cref{sec:mle} that $\max_{\nu} R_t(\nu;\theta^*)$ is not a martingale, which prevents a direct application of Ville's inequality. Laplace's method uses the observation that, under suitable regularity assumptions on a sequence functions $f_n : \bR^d \rightarrow \bR$ with unique maximizer $x_n^*$ and positive definite Hessian $H_n(x) = \frac{\partial^2}{\partial x^2} f_n(x) \in \bR^{d \times d}$, the following asymptotic expansion provides an approximation of the maximizer $f_n(x_n^*)$, 
\begin{align*}
    \int_\Theta h(x) e^{- f_n(x)} dx \sim  \frac{(2\pi)^{d/2} h(x^*_n)}{\sqrt{\det H_n(x_n^*)}} e^{- f_n(x^*_n)}
\end{align*}
To apply this idea to the marginal likelihood that appears in the confidence coefficient in \cref{result:prior_mixing}, assume that $\Theta \subset \bR^d$ and $\mu_0$ admits a density $h(\theta)$ w.r.t. to the Lebesgue measure. Let $\hat \theta_n^\MLE$ be the maximum likelihood estimate, and $I_t(\theta) = \frac{\partial^2}{\partial \theta^2} L_t(\theta)$ the empirical Fischer information matrix. Laplace's methods gives the following approximation
\begin{align*}
\beta_t(\delta) &= \log \frac{1}{\delta} - \log \int \exp(-L_t(\nu))  h(\nu) d\theta\\
&\approx\log \frac{1}{\delta} + L_t(\hat \theta_t^{\MLE})  + \frac{1}{2}\log \det I_t(\hat \theta_n^\MLE) - \frac{d}{2} \log(2\pi) - \log \mu_0(\hat \theta_n^\MLE)  
\end{align*}
Note that the confidence coefficient is smaller (resulting in a smaller confidence set) the more mass $h(\hat \theta^*_n)$ places on the maximizer.
An similar (and perhaps more natural) argument can be made for the maximum a-posteriori estimate.
Unfortunately, making these approximation rigorous is challenging without placing further assumptions on the data generating distribution and function class \citep[e.g.,][]{shun1995laplace}. We will not pursue this any further here, however point out that regret inequalities (\cref{sec:oco}) provide an alternative way to control the error w.r.t.~the MLE, including in finite time.
 
\section{Continuous Exponential Weights}\label{sec:cew}

Continuous exponential weights (\cref{alg:cew}) is a direct generalization of the classical exponential weights algorithm (also known as Hedge) by \cite{freund1997decision,littlestone1994weighted}. For a standard reference, see the book by \cite{cesa2006prediction}. Recall the definition of the log-loss, defined for distributions $\mu \in \sP(\Theta)$,
\begin{align*}
 l_t(\mu) = - \log \left(\int_{\Theta} p_t(y_t|\theta)d\mu(\theta)\right)
\end{align*}
Note that $l_t(\mu)$ is $\eta$-exp-concave for $\eta \leq 1$, since $\exp(-\eta l_t(\mu))$ is concave in $\mu : \sP(\Theta) \rightarrow \bR$. The (continuous) exponential weights algorithm satisfies the following regret bound, that holds for any sequence of $\eta$-exp-concave loss functions.
\begin{theorem}[Regret of Continuous Exponential Weights]\label{thm:cew}
    For any distribution $\rho \in \sP(\Theta)$, and any sequence of $\eta$-exp-concave loss functions $l_1, \dots, l_t$, the regret of the exponential weights algorithms with prior $\mu_0 \in \sP(\Theta)$ and learning rate $\eta$ satisfies
    \begin{align*}
        \sum_{s=1}^t l_s(\mu_{s-1}) - \int L_t(\theta) d\rho \leq \frac{1}{\eta}\KL(\rho, \mu_0)
    \end{align*}
    Moreover, for finite $\Theta$ and any $\nu \in \Theta$,
    \begin{align*}
        \sum_{s=1}^t  l_s(\mu_{s-1}) -  L_t(\nu) \leq \frac{1}{\eta}\log \frac{1}{\rho(\nu)}
    \end{align*}
\end{theorem}

\begin{proof} Denote by $\delta_{\theta} \in \sP(\Theta)$ the Dirac measure on $\theta \in \Theta$. Define the unnormalized measure 
    \begin{align*}
       \tilde \mu_t(d\theta) = \prod_{s=1}^{t} \exp(-\eta l_s(\delta_\theta)) \mu_0(d\theta)
    \end{align*}
    Note that $\mu_t(d\theta) = \frac{\tilde \mu_t(d\theta)}{\tilde \mu_t(\Theta)}$.
To prove the regret bound note that by $\eta$-exp-concavity of $l_t$,
\begin{align*}
    \frac{\tilde \mu_t(\Theta)}{\tilde \mu_{t-1}(\Theta)} &= \int \exp(-\eta l_t(\delta_\theta)) d\mu_{t-1}(\theta) \leq \exp (- \eta  l_t(\mu_{t-1}) )
\end{align*}
Further, for any $\rho \in \sP(\Theta)$, the variational inequality (\cref{lemma:variational-kl}) implies
\begin{align*}
\tilde \mu_t(\Theta) \geq \exp \Big( - \eta \ip{L_n, \rho} - \KL(\rho, \mu_0)\Big)
\end{align*}
Combining the last two inequalities and telescoping, we get
\begin{align*}
    \sum_{s=1}^t  l_s(\mu_{s-1}) -  \int L_s(\theta) d\rho(\theta) \leq \frac{1}{\eta}\KL(\rho, \mu_0)
\end{align*}
This completes the first part of the proof. The second part follows by setting $\rho=\delta_{\nu}$.
\end{proof}

\LinesNumbered
\RestyleAlgo{ruled}
\begin{algorithm2e}[t]
	\DontPrintSemicolon
	\SetAlgoVlined
	\SetAlgoNoLine
	\SetAlgoNoEnd
	\caption{Continuous Exponential Weights} \label{alg:cew}
    \SetKwInput{KwInput}{Input}
    
    \KwInput{Prior $\mu_0 \in \sP(\Theta)$, learning rate $\eta > 0$}

    \For{$t \gets 1, 2. \dots$} {
        \textbf{Predict:} $\mu_{t-1}$\;
        \textbf{Observe:} $x_t, y_t$\;
        \textbf{Receive (log) loss:}$$l_t(\mu_{t-1}) = - \log \left(\int p_t(y_t|\theta)d\mu_{t-1}(\theta)\right)$$\;
        \textbf{Update} $\mu_t(d\theta) \propto \exp\big(-\eta \sum_{s=1}^t l_s(d\theta)\big) \mu_0(d\theta)$ \textbf{:}
        $$\mu_t(d\theta)  = \frac{\exp\big(\eta \log p_t(y_t|\theta)\big) \mu_{t-1}(d\theta)}{\int \exp\big(\eta \log p_t(y_t|\nu)\big) d\mu_{t-1}(\nu)}$$
    }
\end{algorithm2e}

\paragraph{Confidence Sequences using Continuous Exponential Weights} 
Substituting the regret bound from \cref{thm:cew} into \cref{result:posterior_mixing} yields the following $(1-\delta)$-confidence sequence, which holds for any $\cF_t$-adapted sequence $\mu_1,\mu_2, \dots \in \sP(\Theta)$:
\begin{align*}
    C_t  = \left\{ \theta \in \Theta: L_t(\theta) \leq  \log \frac{1}{\delta} + \int L_t(\theta) d\mu_t(\theta) + \KL(\mu_t \| \mu_0)\right\} \,.
\end{align*}
Note that we have recovered the ELBO-confidence set from \cref{result:elbo_confidence_set}. In other words, the regret bound of continuous exponential weights is the sequential analog of the variational inequality \cref{lemma:variational-kl}, and because of the mixing equivalence (\cref{result:mixing-equivalence}), the resulting confidence sets are the same. Unfortunately, for continuous $\Theta$, the bound becomes vacuous when $\rho$ is set to a Dirac measure. This prevents us from using $\mu_t = \delta_{\hat \theta_t^\MLE}$ to derive a confidence set that directly compares to the maximum likelihood estimate. A more careful analysis using additional smoothness assumptions is a possible way forward \citep[c.f.,][]{lee2024unified}.

\section{Misspecified Model Classes}\label{sec:misspecified}

So far, we have assumed that the model class is realizable, that is, there exists a parameter $\theta^* \in \Theta$ such that $p_t(y|\theta^*) = \frac{d\bP_t}{d\xi}$ represents the density of the true data generating distribution. The realizability assumption is used to show that the sequential likelihood ratio is a (super)martingale. The can be significantly relaxed, and there are several ways in which we can construct supermartingales for misspecified model classes, as we elaborate now.

\subsection{Sub-Gaussian Distributions}\label{sec:subgaussian}
In this section, we let $\cY = \bR$ and assume that the true data generating distribution $\bP_t$ is $\sigma$-sub-Gaussian for all $t \geq 1$, that is $\epsilon_t = y_t - \EE[y_t|\cF_{t}]$ satisfies,
\begin{align*}
\EE[e^{\epsilon_t \eta}|\cF_t] \leq e^{\frac{\sigma^2 \eta^2}{2}} \quad \text{for all} \quad \eta \in \bR \,.
\end{align*}
Further assume that we have a parametrized family of mean functions $f_\theta : \cX \rightarrow \bR$, and there exists a $\theta^* \in \Theta$ such that $\EE[y_t|\cF_t] = f_{\theta^*}(x_t)$.
For any $\cF$-adapted sequence of mixing distributions $\mu_0, \mu_1, \dots \in \sP(\Theta)$, define
    \begin{align*}
        E_t(\theta) = \prod_{s=1}^t \frac{ \int \exp(- \frac{1}{2 \sigma^2} (f_{\nu}(x_s) - y_s)^2) d\mu_{s-1}(\nu)}{\exp(- \frac{1}{2 \sigma^2} (f_\theta(x_s) - y_s)^2)} \,.
    \end{align*}
As we will see shortly, $E_t(\theta^*)$ is a $\cF_t$ adapted supermartingale, and $\EE[E_1(\theta^*)] \leq 1$.
Two remarks before we prove the claim. First, if the true data distribution is Gaussian with mean $f_\theta(x)$ and variance $\sigma^2$, then $E_t(\theta)$ is just the sequential marginal likelihood ratio, and the $\sigma$-sub-Gaussian condition holds with equality. Second, the result implies that we can proceed in constructing our confidence set as if the the Gaussian likelihood model was correct, and the coverage results remain true.\looseness=-1
\begin{theorem}
    For any $\cF$-adapted sequence of distributions $\mu_0, \mu_1, \mu_2, \dots$ in $\sP(\Theta)$, define
    \begin{align*}
        C_t = \left\{ \theta \in \Theta:   \sum_{s=1}^t \tfrac{1}{2 \sigma^2}(f_\theta(x_s) - y_s)^2 \leq \log \frac{1}{\delta} + \sum_{s=1}^t \log \int  \exp( -\tfrac{1}{2 \sigma^2} (f_{\nu}(x_s) - y_s)^2) d\mu_{s-1}(\nu)\right\}
    \end{align*}
    Then $C_t$ defines a $(1-\delta)$-confidence sequence.
\end{theorem}
\begin{proof}
We start by showing that $E_t(\theta^*)$ is a super-martingale. Fubini's theorem implies that
\begin{align*}
    \EE[E_t(\theta^*)|\cF_{t-1}] &= E_{t-1}(\theta^*) \int \EE[\exp\left( - \tfrac{1}{2\sigma^2} \big(f_\nu(x_t) - y_t\big)^2 + \tfrac{1}{2\sigma^2} \big(f_{\theta^*}(x_t) - y_t\big)^2 \right) |\cF_t ]  d \mu_{t-1}(\nu)
\end{align*}
From here, we compute the conditional expectation inside the integral. We expand the squares, simplify and substitute $y_t = f_{\theta^*}(x_t) + \epsilon_t$. After a bit of work we arrive at 
\begin{align*}
    &\EE[\exp\left( - \tfrac{1}{2\sigma^2} \big(f_\nu(x_t) - y_t\big)^2 + \tfrac{1}{2\sigma^2} \big(f_{\theta^*}(x_t) - y_t\big)^2 \right) |\cF_t ] \\
    &= \exp\left(- \tfrac{1}{2\sigma^2} \big(f_{\hat \theta_{t-1}}(x_t) - f_{\theta^*}(x_t)\big)^2 \right) \EE[\exp\left( \epsilon_t \cdot \tfrac{1}{\sigma^2} \big(f_{\hat \theta_{t-1}}(x_t) - f_{\theta^*}(x_t)\big) \right) |\cF_t] \,.
\end{align*}
Next, we use that $\epsilon_t$ is $\sigma^2$-sub-Gaussian, which, by definition, implies that 
\begin{align*}
    \EE[\exp\left( \epsilon_t \cdot \tfrac{1}{\sigma^2} \big(f_{\hat \theta_{t-1}}(x_t) - f_{\theta^*}(x_t)\big) \right) ] \leq \exp\left(\tfrac{1}{2\sigma^2} \big(f_{\hat \theta_{t-1}}(x_t) - f_{\theta^*}(x_t)\big)^2 \right) \,.
\end{align*}
We conclude that $\EE[E_t(\theta^*)|\cF_{t-1}] \leq  E_{t-1}(\theta^*)$ and $\EE[E_1(\theta^*)] \leq 1$. The claim follows using Ville's inequality.
\end{proof}




\subsection{Convex Model Classes}\label{sec:convex}
We return to the original definition of the model class as a parameterized family of conditional densities $\cM = \{ p_\theta(y|x) : \theta \in \Theta\}$. Moreover, assume that there exists a conditional density $p^*(y|x)$ such that for all $t \geq 1$, $\frac{d\bP_t}{d\xi} = p^*(\cdot|x_t) d\xi$. Crucially, we do \emph{not} require that $p^* \in \cM$. 

Throughout this section, we make the assumption that $\cM$ is convex. Note that convexity is required in the space of distributions, i.e., all finite mixtures of densities in $\cM$ are contained in $\cM$. In general, convexity of $\Theta$ does not imply convexity of $\cM$. Nevertheless, there are many examples of convex model classes, including, for example, all finite mixtures of any family of distributions. 

Our main tool is the \emph{reverse information projection} theorem by \citet{li1999estimation}, see also \citet{lardy2024reverse}. Applied to our setup, the theorem states that for any sequence $q_n \in \cM$ such that $$\lim_{n \rightarrow \infty} \KL(p^*\|q_n) = \inf_{q \in \cM} \KL(p^*\|q) < \infty\,,$$ there exists a unique (sub-)probability measure $q^* d\xi$ such that $\KL(p^*\|q^*) = \inf_{q \in \cM} \KL(p^*\|q)$. Moreover, the reverse information projection theorem shows that for any $q_\theta \in \cM$,
\begin{align}
    \EE[\frac{p_\theta(y_t|x_t)}{q^*(y_t|x_t)}\big | \cF_t] \leq 1 \,. \label{eq:rips-e-value}
\end{align}
A technical condition is required to ensure that the limiting element $q^*$ is contained $\cM$. If we require that $\cY$ is a complete separable metric space, and $\cM$ is sequentially compact (with respect to the weak topology), then Prokhorov's theorem implies that $q^* \in \cM$. A similarly flavoured result (stated without mixing distributions) is by \citet[Proposition 7]{wasserman2020universal}

\begin{theorem}[Convex Model Classes] Assume that $\cM$ is convex and there exists $q^* \in \cM$ such that $\KL(p^*\|q^*) = \inf_{q \in \cM} \KL(p^*\|q)$. Then the sequential likelihood mixing confidence set, defined for any $\cF_t$-adapted sequence of distributions $\mu_0, \mu_1, \mu_2, \dots$ in $\sP(\Theta)$ (see \cref{result:posterior_mixing}), defines a $(1-\delta)$-confidence sequence for $q^* \in \cM$, i.e., $\bP[q^* \in C_t, \forall t \geq 1] \geq 1-\delta$.
\end{theorem}
\begin{proof}
Define the sequential marginal likelihood ratio w.r.t. to a conditional density $q(\cdot|x)$,
\begin{align*}
J_t(q) = \prod_{s=1}^t\frac{\int p_\nu(y_s|x_s) d\mu_{s-1}(\nu)}{q(y_s|x_s)}
\end{align*}
The reverse information projection theorem, specifically \cref{eq:rips-e-value}, implies that $J_t(q^*)$ is a non-negative supermartingale with $\EE[J_1] \leq 1$. The theorem follows using Ville's inequality.
\end{proof}

\section{Tempered Likelihood Ratios}\label{sec:tempered}

The Bayesian update rule can be generalized by introducing a \emph{temperature} parameter $\beta > 0$,
\begin{align*}
    \mu_t(\theta) \propto \prod_{s=1}^t p_s(y_s|\theta)^\beta \mu_0(\theta) \,.
\end{align*}
The generalized update rule has been studied under various names, e.g. as \emph{fractional posteriors} \citep{bhattacharya2019bayesian}, \emph{powered likelihoods} \citet{holmes2017assigning} and \emph{tempered posteriors} \citep{alquier2020concentration}, often in the context of adding robustness to misspecification in the Bayesian model \citep{grunwald2012safe}. The result presented below are similar in spirit to the work by \citet{zhang2006varepsilon}, who studies complexity bounds for density estimation in the classical i.i.d.~setting.

\paragraph{Divergences} A few more definitions will be useful, we follow the exposition of \citet{van2014renyi}. Let $p,q \in \sP(\cY)$ be two distributions over the observation space $\cY$. We assume that $p,q$ admit densities w.r.t. a common base measure. We define the Rényi divergence for  $p,q$ and parameter $0 \leq \zeta \leq 1$,
\begin{align*}
    D_\zeta(p\|q) = \frac{1}{\zeta - 1} \log \int p(x)^{\zeta} q(x)^{1- \zeta} dx \,.
\end{align*}
For $0< \zeta < 1$, it holds that
\begin{align*}
   (1-\zeta)  D_\zeta(p\|q) = \zeta D_{1-\zeta}(q\|p)\,.
\end{align*}
Moreover, the Hellinger distance is given by 
\begin{align*}
    H^2(p\|q) = \int \big(\sqrt{p(x)} - \sqrt{q(x)}\big)^2 dx\,.
\end{align*}
Hellinger and Rényi divergences satisfy the following relation:
\begin{align*}
    \frac{1}{2} H^2(p\|q) \leq D_{1/2}(p\|q)\,.
\end{align*}
For notational convenience, we define
\begin{align*}
    D_{\zeta,t}(\theta\|\nu) &:= D_\zeta(p_t(\cdot|\theta)\|p_t(\cdot|\nu)) \,,\\
    H^2_t(\theta\|\nu) &:= H^2(p_t(\cdot|\theta)\|p_t(\cdot|\nu)) \,.\\
\end{align*}

\subsection{Tempered Confidence Sequences}
Let $\hat \theta_0, \hat \theta_1, \dots$ be an $\cF_t$-adapted sequence of estimators.
Define the \emph{tempered} log-ratio, 
\begin{align}
    A_t^\beta(\theta) = - \beta \log \frac{p_t(y_t|\hat \theta_{t-1})}{p_t(y_t|\theta)} \,.
\end{align}
In particular, by applying the exponential function, we get the \emph{tempered likelihood ratio},
\begin{align*}
\exp(-A_t^\beta(\theta)) = \left(\frac{p_t(y_t|\hat \theta_{t-1})}{p_t(y_t|\theta)}\right)^\beta
\end{align*}
As a side remark, Jensen's inequality implies that 
$\bE[\exp(-A_t^\beta(\theta^*))] \leq  \bE\big[\frac{p_t(y_t|\hat \theta_{t-1})}{p_t(y_t|\theta^*)}\big]^\beta = 1$.
Hence $\prod_{s=1}^t A_t^\beta(\theta^*)$ is a super-martingale, and all results presented in the main paper continue to hold for the tempered ratio. However, we can strengthen the construction by enforcing the martingale property. Define
\begin{align*}
M_t(\theta) = \frac{\exp(- \sum_{s=1}^t A_t^\beta(\theta))}{\textstyle \prod_{s=1}^t \int \exp(- A_t^\beta(\theta)) p_t(y|\theta) dy} \,.
\end{align*}
By definition, $M_t(\theta^*)$ is a non-negative martingale. 
Hence, Ville's inequality implies that
\begin{align*}
   \bP\left[- \sum_{s=1}^t \log \int \exp(- A_s^\beta(\theta^*)) p_s(y|\theta^*) dy - \sum_{s=1}^t A_s(\theta^*) \geq \log \frac{1}{\delta}\right] \leq \delta \,.
\end{align*}
Finally, we note that 
\begin{align*}
   - \log \int \exp(- A_t^\beta(\theta)) p_t(y|\theta) dy = (1 - \beta) D_{\beta,t}(\hat \theta_{t-1} \| \theta)= \beta D_{1-\beta,t}(\theta\|\hat \theta_{t-1}) \,.
\end{align*}

\begin{theorem}[Tempered Confidence Set]\label{result:tempered} Let $\hat \theta_0, \hat \theta_1, \dots$ be an $\cF_t$-adapted sequence of estimators. Define the log regret $\Lambda_t(\theta) = - \sum_{s=1}^t \log p_s(y_s|\hat \theta_{s-1}) - L_t(\theta)$ for $\theta \in \Theta$, and
    \begin{align*}
        C_t^\beta = \left\{ \theta \in \Theta : \sum_{s=1}^t D_{1-\beta,s}(\theta\|\hat \theta_{s-1}) - \Lambda_t(\theta)\leq \frac{1}{\beta} \log \frac{1}{\delta}\right\} \,.
    \end{align*}
    Then $C_t^\beta$ defines a $(1-\delta)$-confidence sequence. Moreover, define
    \begin{align*}
        C_t^H = \left\{ \theta \in \Theta :\sum_{s=1}^t  H_s^2(\theta\|\hat \theta_{s-1}) - \Lambda_t(\theta)\leq 2 \log \frac{1}{\delta}\right\} \,.
    \end{align*}
    Then $C_t^H$ defines a $(1-\delta)$-confidence sequence.
\end{theorem}
As a consequence of the result, assume that the estimation sequence is constructed to achieve bounded regret, $\Lambda_t(\theta) \leq B_t$ for a predictable sequence $B_t \geq 0$, typically related to the complexity of the model class (c.f.,~\cref{sec:oco} and \cref{result:regret}), see also \citet{zhang2006varepsilon}. Then \cref{result:tempered} provides confidence sequences that only depend on the Hellinger or Renyi divergences. The advantage is, that unlike the likelihood ratios, the divergence is predictable quantity under the filtration $\cF_t$, hence can be used in sequential decision making setting to control the state $x_t$. This is one of the reasons why similar bounds have recently gained interest in the sequential decision-making literature \citep[e.g.,][]{chen2022unified,foster2021statistical,foster2023tight,wagenmaker2023instance}. In the next section, we provide another application to sparse estimation.




\subsection{Online Linear Prediction}\label{sec:sparse}
Recall the sequential linear regression setting from \cref{sec:bayes}. Specifically, assume that $\Theta \subset \bR^d$, and Gaussian likelihood $p_t(y|\theta) \sim \cN(\ip{\theta, x_t}, 1)$ for $\theta \in \Theta$ and covariates $x_t \in \bR^d$. Everything in this section generalizes to $\sigma$-sub-Gaussian noise distribution using the same arguments as in \cref{sec:subgaussian}. In a slight generalization to earlier results, assume that an online learning algorithm produces predictions $\hat y_t \in \cY$ (opposed to $\hat \theta_t \in \Theta$), in way such that the regret satisfies the following bound for any $\theta \in \Theta$,
\begin{align*}
    \Lambda_t(\theta) = \sum_{s=1}^t \frac{1}{2}(\hat y_{s-1} - y_s)^2 - \frac{1}{2}(\ip{\theta, x_s} - y_s)^2 \leq B_t(\theta) \,.
\end{align*}
For a concrete instantiation in the linear setting, we refer to the famous Vovk-Azoury-Warmuth forecaster \citep{vovk1997competitive,azoury2001relative}. We remark that the bound has been further generalized to non-parametric settings by \citet{rakhlin2014online}. We are now in place to generalize the online-to-confidence set conversion by \citet{abbasi2012online}.

\begin{lemma}[Online-To-Confidence Convertion] Assume the sequential linear regression setup defined above with a known bound $\Lambda_t(\theta^*) \leq B_t$. For any $\cF_t$-adapted sequence $\hat y_0, \hat y_1, \hat y_2, \dots,$ and  $0 < \beta < 1$, let
    \begin{align*}
        C_t = \left\{ \theta \in \Theta : \sum_{s=1}^t \frac{1}{2} \big(\hat y_{s-1} - \ip{\theta, x_s}\big)^2 \leq \frac{1}{\beta - \beta^2 } \log \frac{1}{\delta} + \frac{\beta}{\beta - \beta^2} B_t \right\}
    \end{align*}
    Then $C_t$ defines a $(1-\delta)$-confidence sequence.
\end{lemma}
\begin{proof}
The plan is to compute  $- \log \int \exp(- A_t^\beta(\theta)) p_t(y|\theta) dy$ for the Gaussian distribution, and then invoke \cref{result:tempered}. It is useful to note that for Gaussian distributed variable $\epsilon \sim \cN(0, \sigma^2)$ the moment generating function is $\bE[\exp(t \epsilon)] = \exp(\frac{1}{2}\sigma^2 t^2)$.
A short calculation reveals that 
\begin{align*}
    & \int \exp\big(- A_t^\beta(\theta)\big) p_t(y|\theta) dy\\
    &= \int \exp\Big(-  \frac{\beta}{2} \big(\hat y_{s-1} - y\big)^2 + \frac{\beta}{2} \big(\ip{\theta,x_t} - y\big)^2 \Big)  p_t(y|\theta) dy\\
    &= \int \exp\Big(-  \frac{\beta}{2} \big(\hat y_{s-1} - {\theta,x_t}\big)^2 + \beta \epsilon_t \big(\ip{\theta,x_t} - \ip{\theta,x_t}\big) \Big)  p_t(y|\theta) dy\\
    &= \exp\Big(-  \frac{\beta - \beta^2}{2} \big(\hat y_{s-1} - \ip{\theta,x_t}\big)^2 \Big) \,.
\end{align*}
Hence,
\begin{align*}
    & - \log \int \exp\big(- A_t^\beta(\theta)\big) p_t(y|\theta) dy =  \frac{\beta - \beta^2}{2} \big(\hat Y_{s-1} - \ip{\theta,x_t}\big)^2 \,.
\end{align*}
Lastly, we make use of the assumption that $\Lambda_t(\theta^*) \leq B_t$. The result follows by intersecting the confidence set $C_t^\beta$ from \cref{result:tempered} with  $\{\theta \in \Theta : \Lambda_t(\theta) \leq B_t\}$.
\end{proof}

\paragraph{Sparse Linear Regression} We make the additional assumption that the true parameter $\theta^* \in \Theta$ is $k$-sparse, i.e. $\|\theta^*\|_0 \leq k$. The key insight is that \citet{gerchinovitz2011sparsity} provides an online algorithm, producing the sequence $\hat y_0, \hat y_1, \dots$, for which $B_t(\theta^*) \leq  \cO( k \log(t))$. We compare to the result by \citet{abbasi2012online} in the same setting. For $\beta = \frac{1}{2}$, our result reads
\begin{align*}
   C_t = \left\{ \theta \in \Theta : \sum_{s=1}^t \frac{1}{2} (\hat y_{s-1} - \ip{\theta, x_s})^2 \leq 4\log \frac{1}{\delta} + 2 B_t(\theta^*) \right\} \,.
\end{align*}
In comparison, the result by \citet[Theorem 1]{abbasi2012online} reads, in our notation,
\begin{align*}
C_t = \left\{ \theta \in \mathbb{R}^d : \sum_{s=1}^{t}  \frac{1}{2}(\hat y_{s-1} - \langle \theta, x_t \rangle)^2 \leq 16 \log \frac{1}{\delta} + \frac{1}{2} + 2B_t(\theta^*) + 16\log (\sqrt{8} + \sqrt{1 + 2 B_t(\theta^*)}) \right\}
\end{align*}
The additional terms stem from using recursive inequality on the square-error, which we avoid using the more direct argument via the tempered likelihood ratio. This demonstrates the benefit of the sequential likelihood framework for deriving confidence sets via the online-to-confidence set conversion.

\end{document}